\documentclass[letterpaper, 10 pt, conference]{ieeeconf}  % Comment this line out if you need a4paper

\IEEEoverridecommandlockouts                              % This command is only needed if 
\usepackage{graphics} % for pdf, bitmapped graphics files
\usepackage{epsfig} % for postscript graphics files
\usepackage{times} % assumes new font selection scheme installed
\usepackage{amsmath} % assumes amsmath package installed
\usepackage{amssymb}  % assumes amsmath package installed

\usepackage{amsthm}
\newtheorem{theorem}{Theorem}
\newtheorem{lemma}{Lemma}

\usepackage{multicol}
\usepackage{multirow}
\usepackage[ruled]{algorithm2e}

\usepackage{textcomp}
\usepackage{hyperref}
\usepackage{cleveref}

\usepackage[utf8]{inputenc}
\usepackage[T1]{fontenc}
\usepackage{xcolor}
\usepackage{float}
\usepackage{gensymb}
\usepackage{bm}
\usepackage{autobreak}
\usepackage{svg}
\usepackage{mathrsfs}
\usepackage{tabularx}
\usepackage{array, makecell}
\usepackage{graphicx}
\usepackage{cellspace}

\usepackage{comment}

\usepackage{soul}

\title{\LARGE \bf
Kinetostatic Path Planning for Continuum Robots By Sampling on Implicit Manifold
}

\author{Yifan Wang$^{1}$ and Yue Chen$^{2}$% <-this % stops a space
\thanks{Research reported in this publication is    supported by the National Institute of Biomedical Imaging And Bioengineering of the National Institutes of Health under Award Number R01EB034359. The content is solely the responsibility of the authors and does not necessarily represent the official views of the National Institutes of Health.  Corresponding author: Yue Chen. }
\thanks{$^{1}$Yifan Wang is with the Department of Mechanical Engineering, 
        Georgia Institute of Technology, Atlanta 30332, USA
        {\tt\small wangyf@gatech.edu}}%
\thanks{$^{2}$Yue Chen is with the Department of Biomedical Engineering, 
        Georgia Institute of Technology, Atlanta 30332, USA
        {\tt\small yue.chen@bme.gatech.edu}}%
}

\begin{document}

\maketitle
\thispagestyle{plain}
\pagestyle{plain}

%%%%%%%%%%%%%%%%%%%%%%%%%%%%%%%%%%%%%%%%%%%%%%%%%%%%%%%%%%%%%%%%%%%%%%%%%%%%%%%%
\begin{abstract}

Continuum robots (CR) offer excellent dexterity and compliance in contrast to rigid-link robots, making them suitable for navigating through, and interacting with, confined environments. However, the study of path planning for CRs while considering external elastic contact is limited. The challenge lies in the fact that CRs can have multiple possible configurations when in contact, rendering the forward kinematics not well-defined, and characterizing the set of feasible robot configurations as non-trivial. In this paper, we propose to solve this problem by performing quasi-static path planning on an implicit manifold. We model elastic obstacles as external potential fields and formulate the robot statics in the potential field as the extremal trajectory of an optimal control problem obtained by the first-order variational principle. We show that the set of stable robot configurations is a smooth manifold diffeomorphic to a submanifold embedded in the product space of the CR actuation and base internal wrench. We then propose to perform path planning on this manifold using AtlasRRT*, a sampling-based planner dedicated to planning on implicit manifolds. Simulations in different operation scenarios were conducted and the results show that the proposed planner outperforms Euclidean space planners in terms of success rate and computational efficiency.

\end{abstract}

%%%%%%%%%%%%%%%%%%%%%%%%%%%%%%%%%%%%%%%%%%%%%%%%%%%%%%%%%%%%%%%%%%%%%%%%%%%%%%%%
\section{INTRODUCTION}

Continuum robots (CR) are continuously deformable structures with high dexterity and passive compliance, making them widely investigated for applications involving contact or interaction with the environment \cite{walker2013review,jessica_review}. To this end, theories of CR modeling under external loads have been widely studied, and it is well-established, both theoretically and experimentally, that the Cosserat rod model is accurate in describing the mechanics of a slender CR \cite{dupont2010CTR,rucker-ctr,rucker-tendon-driven,boyer2023OCP}. Control of CRs under external loads have also been studied, including stiffness modulation \cite{mahvash2011stiffness} and force/position hybrid control \cite{bajo2016motion/force}. While most of these studies focused on the scenario where the robot is only under a tip load, a few studies have considered multiple contacts along the robot body \cite{goldman2014compliant,zhang2019contact-control}. However, these studies have all focused on local optimization-based control. 

Conversely, global path planning methods for CRs have been developed for navigation through confined environments and avoidance of unstable configurations. A majority of these studies used sampling-based methods \cite{bergeles2013stable_path,wu2015centerline,kuntz2019CTR_point_cloud,hoelscher2021needle_planning,meng2022workspaceRRT}, while others leveraged heuristics, such as follow-the-leader motions \cite{mohammad2021follow-the-leader}. However, these methods avoid environmental contact, limiting their use from tasks that could exploit their compliance. A few works have considered contacts during planning. In \cite{li2016grasp}, point contacts are utilized to progressively generate a wrapping path for CR grasping. In \cite{greer2020growing_robot}, contacts facilitate path changes of a soft growing robot for navigation.
However, these studies used geometric approaches and lacked mechanics informed planning. Further, these studies focused on planar scenarios.

Path planning for CRs with contacts entails two major problems. First, when a CR is under external loads, its configuration (shape) is not fully determined by the actuation but also the load. This is illustrated in Fig. \ref{fig.example}, where the same set of CR actuation values results in different configurations when there is external contact. Therefore, unlike rigid-link robots, the mapping from actuation to configuration in CR is not well-defined, and plans in the actuation space may not cover all possible configurations. Second, planning directly in the configuration space is not trivial. The shape of a CR is described by a continuous curve, which belongs to a subset of an infinite-dimensional functional space. Although the shape can be approximated by functional basis interpolation to generate a finite-dimensional configuration space \cite{boyer2021strain_parameterization,sadati2022reduced-order}, it is non-trivial to check whether a configuration is achievable under environmental contact. Due to the above reasons, applying existing robot planning methods to CRs with contact is not straightforward and remains understudied.

\begin{figure}[t!]
    \centering
    \includegraphics[width = 0.5\linewidth]{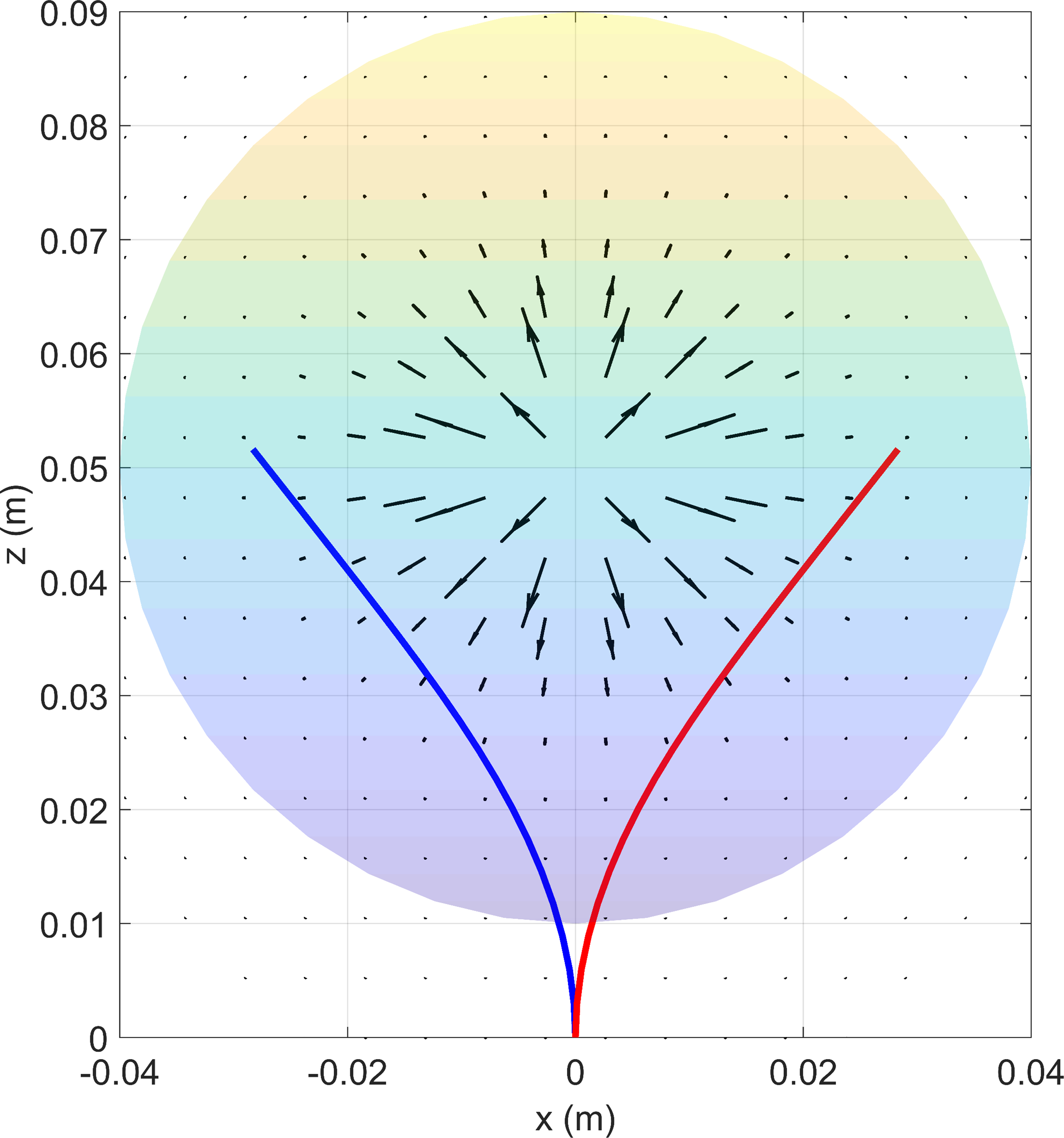}
    \vspace{-3mm}
    \caption{A CR with the same actuation values, but different configurations. The vector field represents a potential field exerting distributed forces.} 
    \label{fig.example}
    \vspace{-7mm}
\end{figure}

A widely studied problem that poses similar challenges is the manipulation of an elastic rod. In this problem, a rod is grasped at each end by separate robotic manipulators and the motion of the manipulators are planned to provide a desired rod shape. 
% \yc{we can probably cite this: Continuum Reconfigurable Parallel Robots for Surgery: Shape Sensing and State Estimation With Uncertainty, if we have space eventually}. 
Similar to the CRs, the rod is described by a continuous curve that can have different configurations with the same end poses. In the  work by Bretl and McCarthy \cite{bretl2014kirchhoff}, it is shown that the set of stable rod configurations in free space is a smooth 6-dimensional manifold parameterized by a single global chart that corresponds to the internal wrench of the rod at one end. It is then straightforward to perform path planning in the Euclidean space of the internal wrench. In later works, it is shown that this is also true for rods in external smooth potential fields, such as the gravity field \cite{borum2018manipulation,wu2022plan_in_force}.

Inspired by the solution to the problem of rod manipulation planning, we propose a path planning method for CRs in environments with elastic objects. We model the contact forces between the robot and the objects as conservative forces generated by smooth potential fields. Following a derivation that is modified from \cite{bretl2014kirchhoff,borum2018manipulation}, we show that the set of stable configurations of the CR in a given potential field is a smooth manifold that is diffeomorphic to a submanifold embedded in the product space of the CR actuation and backbone internal wrench at the base. This submanifold is implicitly defined and its dimension equals the dimension of the actuation space, i.e. degrees of freedom (DoF). We then propose to perform path planning on this manifold using a variant of Rapid-exploring Random Tree (RRT) called AtlasRRT* \cite{jaillet2012atlasRRTstar}, which is designed to conduct sampling-based planning on implicitly defined manifolds. Simulations were conducted to compare the proposed method against RRT* in the actuation space and in the Euclidean product space of the CR actuation and base internal wrench.

This paper is organized as follows: Sec. \ref{sec_preliminaries} provides an overview of the differential geometry and optimal control theories we used in our modeling and analysis. The robot mechanics model and our proof of the aforementioned result are presented in Sec. \ref{sec_mechanics}. Sec. \ref{sec_planning} provides details of the planning algorithm we used. The simulation results are presented in Sec. \ref{sec_results}, followed by the conclusion in Sec. \ref{sec_conclusions}. 

% \vspace{-1mm}

\section{Preliminaries} \label{sec_preliminaries}

To model CRs inside external potential fields, we utilize  optimal control theory, which has demonstrated the Cosserat rod model from an energy perspective and also provides stability conditions \cite{boyer2023OCP,ha2016stability}. 
We then analyze the geometric structure of the obtained model.
In this section, we briefly introduce the theoretical background for the modeling and analysis later. More details can be found in \cite{lee_introduction_2012,marsden_introduction_1999}.

\subsection{Differential Geometry}

Let $\mathcal{M}$ be a topological $n$-manifold, a chart on $\mathcal{M}$ is a pair $(U,\varphi)$, where $U$ is an open subset of $\mathcal{M}$ and $\varphi:~U\rightarrow \varphi(U)\in \mathbb{R}^n$ is a homeomorphism that maps $U$ to an open subset of $\mathbb{R}^n$.
An atlas of $\mathcal{M}$ is a collection of charts $\{(U_i,\varphi_i)\}$ such that $\{U_i\}$ is a cover of $\mathcal{M}$.

A diffeomorphism is a smooth map that is bijective and has a smooth inverse. Two charts $(U_i,\varphi_i)$ and $(U_j,\varphi_j)$ are smoothly compatible if either $U_i\bigcap U_j = \emptyset$ or the composed map $\varphi_j\circ\varphi_i^{-1}:~\varphi_i(U_i\bigcap U_j)\rightarrow\varphi_j(U_i\bigcap U_j)$ is a diffeomorphism. An atlas is called a smooth atlas if any pair of charts in it are smoothly compatible. A smooth atlas is called a smooth structure on $\mathcal{M}$ if any chart that is smoothly compatible with charts in the atlas is also included in the atlas. If $\mathcal{M}$ is equipped with a smooth structure, it is a smooth manifold.

Diffeomorphisms between smooth manifolds preserve smooth structures. Let $\mathcal{M}$ and $\mathcal{N}$ be smooth manifolds and define a map $f:~\mathcal{M}\rightarrow \mathcal{N}$. $f$ is smooth if $\psi_j\circ f\circ\varphi_i$ is a smooth map for all charts $(U_i,\varphi_i)$ on $\mathcal{M}$ and $(V_j,\psi_j)$ on $\mathcal{N}$. If $f$ is a diffeomorphism, then it transports the smooth structure $\{(U_i,\varphi_i)\}$ on $\mathcal{M}$ to a smooth structure $\{(f(U_i),\varphi_i\circ f^{-1})\}$ on $\mathcal{N}$. Indeed, for any pair of such charts on $\mathcal{N}$ with non-empty intersection, $\varphi_j\circ f^{-1} \circ (\varphi_i\circ f^{-1})^{-1} = \varphi_j\circ\varphi_i^{-1}$ is a diffeomorphism. 

\subsection{Lie Group}

A Lie group is a group that is also a smooth manifold, where the group multiplication and inversion are smooth maps. Its corresponding Lie algebra is the tangent space of the manifold at the identity element of the group. Let $g(s)$ be a trajectory on the Lie group $SE(3)$ parameterized by $s$, its body velocity $\xi = [\omega^T~v^T]^T \in \mathbb{R}^6$, where $\omega$ is the angular velocity and $v$ is the linear velocity. The corresponding vector in the Lie algebra $\mathfrak{se}(3)$ is $\Hat{\xi}$, which in local coordinates writes
\begin{equation}
    \Hat{\xi} = \begin{bmatrix}
        \Hat{\omega} & v \\ 0_{1\times3} & 0
    \end{bmatrix},~\Hat{w} = \begin{bmatrix}
        0 & -\omega_3 & \omega_2\\
        \omega_3 & 0 & -\omega_1\\
        -\omega_2 & \omega_1 & 0
    \end{bmatrix}
\end{equation}
Here, the $\Hat{\cdot}$ operator is abused to also represent $\mathbb{R}^3\rightarrow\mathfrak{so}(3)$. It follows that $g^\prime = g\Hat{\xi}$, where $g^\prime$ is the derivative of $g$ w.r.t. $s$.
Now, consider a $\delta$-variation of $g$ over a parameter independent of $s$, we have $\delta g = g\delta\Hat{\zeta}$, where $\delta\zeta$ represents the body twist of the variation. Due to the independence of parameters, the commutation relation $\delta(g^\prime) = (\delta g)^\prime$ holds, which gives \cite{marsden_introduction_1999}
\begin{equation}\label{twist_variation}
    \delta \xi = \delta \zeta^\prime + ad_{\xi}\delta\zeta,~ad_{\xi} = \begin{bmatrix}
        \Hat{w} & 0_{3\times3}\\
        \Hat{v} & \Hat{w}
    \end{bmatrix}
\end{equation}

\section{Robot Mechanics and Geometric Analysis} \label{sec_mechanics}

In this section, we apply optimal control theory to model slender CRs and present our main results based on the model.

\subsection{Mechanics Model}

Consider the robot backbone to be an elastic rod. For each cross section of the robot backbone, a material coordinate frame is attached such that its $XY$-plane coincides with the cross section and its $Z$-axis points along the direction of increasing $s$. The material frame is described by $g \in SE(3)$, and its differential kinematics as it moves along the robot reference arc-length $s$ is given by $g^\prime = g\Hat{\xi}$, where  $\xi\in\mathbb{R}^6$ is the body twist of $g$ representing the material strain. 

We assume that all internal and external forces of the robot are conservative, such that they are generated by their respective potential fields. Assuming that the backbone has natural body twist $\xi_0$ and the strains of the backbone are small, its internal elastic potential energy is $\frac{1}{2}\int_0^{l}\epsilon^TC\epsilon\mathrm{d}s$, where $l$ is the original length of the robot, $\epsilon = \xi- \xi_0$ is the material strain, and $C = \mathrm{diag}(EI, EI, GJ, GA, GA, EA)$ is the matrix of stiffness. This potential energy induces an internal wrench $C\epsilon$ that tends to restore the natural shape of the backbone. Adopting the notation in \cite{boyer2023OCP}, the wrench generated by the internal actuating mechanism of the robot to prevent the shape restoration is denoted by $\Lambda_{ad}$, which has the corresponding internal potential field $\epsilon^T\Lambda_{ad}$. We then model the elastic obstacles in the environment as a smooth potential field $\mathcal{U}: SE(3)\rightarrow\mathbb{R}$, such that it induces an external wrench $W$ at $g$ satisfying
\vspace{-1mm}
\begin{equation}
    \delta\zeta^TW(g) = -\delta \mathcal{U}|_g
    \vspace{-2mm}
\end{equation}
We also model any concentrated external tip wrench $W_+$ as generated by a smooth potential field $\mathcal{U}^+: SE(3)\rightarrow\mathbb{R}$.

By the principle of minimum potential energy, a quasi-static configuration of the robot is an extremal of the following optimal control problem with the state trajectory $g(s)$ in the Lie group $SE(3)$ and parameterized by $s$:
\begin{equation}
    \begin{aligned}  \min_{u}\int_0^{l}&(\frac{1}{2}u^TCu + u^T\Lambda_{ad} + \mathcal{U}(g))\mathrm{d}s+\mathcal{U}^+(g(l))\\ 
    \mathrm{s.t.}~g^\prime &= g\Hat{\xi} = g(u + \xi_0)^{\wedge},~
    g(0) = e
    \end{aligned}
    \vspace{-1mm}
\end{equation}
where the cost is the total potential energy, the control is the strain $\epsilon = u$, and $e$ is the identity element of $SE(3)$. 

In \cite{bretl2014kirchhoff}, Lie-Poisson reduction was performed to reduce the extremal trajectory to a curve on the dual Lie algebra $\mathfrak{se}^*(3)$, removing the dependence on $g$, since the Hamiltonian for their problem is left-invariant on $SE(3)$. However, as discussed in \cite{borum2018manipulation}, potential fields that are dependent on $g$ render the Hamiltonian non-left-invariant, and the trajectory cannot be reduced to a lower dimension. We thus directly apply the principle of variational calculus as in \cite{boyer2023OCP}. Reducing the system dynamics constraint to Lie algebra as $\xi = u + \xi_0$, we arrive at the augmented cost functional
\vspace{-1mm}
\begin{equation}
    \int_0^{l}(\frac{1}{2}u^TCu + u^T\Lambda_{ad} + \mathcal{U} + (\xi - u - \xi_0)^T\Lambda)\mathrm{d}s+\mathcal{U}^+
    \vspace{-1mm}
\end{equation}
where $\Lambda$ is the vector of Lagrange multipliers. The first-order variation of the cost functional caused by the variation $\delta u$ is 0 for an extremal trajectory:
\begin{equation}
    \int_0^{l}(\delta u^T(Cu + \Lambda_{ad}) + \delta\mathcal{U} + (\delta\xi - \delta u)^T\Lambda)\mathrm{d}s+\delta\mathcal{U}^+ = 0
\end{equation}
where $\delta\xi$ comes from the variation in $g$ that is given by (\ref{twist_variation}). Substituting (\ref{twist_variation}) into the above equation yields
\vspace{-1mm}
\begin{equation}
\begin{aligned}
    \int_0^{l}(\delta u^T&(Cu + \Lambda_{ad}) + \delta\mathcal{U})\mathrm{d}s + \\
    &\int_0^{l}(\delta\zeta^{\prime T} + \delta\zeta^T ad_{\xi}^T - \delta u)^T\Lambda\mathrm{d}s+\delta\mathcal{U}^+ = 0
\end{aligned}
\vspace{-1mm}
\end{equation}
Integrating by parts to eliminate $\delta\zeta^\prime$ yields
\begin{equation}
\begin{aligned}
    \int_0^{l}&(\delta u^T(Cu + \Lambda_{ad} - \Lambda) - \delta\zeta^T(\Lambda^{\prime} - ad_{\xi}^T\Lambda + W))\mathrm{d}s\\
    & +\delta\zeta^T(l)(\Lambda(l)-W_+(g(l))) = 0
\end{aligned}
\end{equation}

Making the coefficients of $\delta\zeta$ vanish and combining the system dynamics constraint, we arrive at a system of ordinary differential equation (ODE)
\begin{equation}\label{canonical_eqs}
    \begin{aligned}
        g^\prime &= g(\xi_0 + u)^\wedge\\
        \Lambda^\prime &= ad_{\xi}^T\Lambda - W
    \end{aligned}
\end{equation}

Making the coefficients of $\delta u$ vanish gives
\vspace{-1mm}
\begin{equation}\label{constitutive_law}
    \Lambda =  Cu +  \Lambda_{ad}
    \vspace{-1mm}
\end{equation}
Note that for a slender CR with a small radius and small strains, the dependence of $\Lambda_{ad}$ on $u$ is negligible as remarked in \cite{boyer2023OCP}, therefore (\ref{constitutive_law}) is still valid.

% The initial conditions are $\mu(0) = \sum_{i=1}^6\Lambda_iP_i$ and $\psi^{(i)}(0) = \chi^{(i)}_0q_0$, 
Finally, the boundary condition is given by setting the coefficients of $\delta\zeta(l)$ to 0:
\vspace{-1mm}
\begin{equation}\label{B.C.}
    \Lambda(l) - W_+(g(l)) = 0
    \vspace{-1mm}
\end{equation}

Equations (\ref{canonical_eqs})-(\ref{B.C.}) form a boundary value problem (BVP) that describes the CR mechanics. We note that (\ref{canonical_eqs}) is parameterized by the actuation variables $\tau$ via the actuation wrench $\Lambda_{ad}(\tau)$, which we assume to be a smooth map. To solve this BVP, we use the direct shooting method to solve for the unknown initial condition $\lambda := C\epsilon(0)$, which is the backbone internal wrench at the robot base.

\subsection{Geometric Analysis}

We show that the set of stable configurations of the robot is a smooth manifold. Let $C^\infty([0,l],SE(3)\times U)$ be the set of all smooth maps $(g,u):~[0,l]\rightarrow SE(3)\times U$ under the smooth topology. Let $\mathcal{C}\subset C^\infty([0,l],SE(3)\times U)$ be the subset of all $(g,u)$ that satisfy the BVP (\ref{canonical_eqs}) - (\ref{B.C.}). Since the autonomous ODEs (\ref{canonical_eqs}) are smooth in $\lambda$, $g$, and parameter $\tau$, any $(g,u)\in\mathcal{C}$ is uniquely and smoothly determined by the choice of $(\lambda,\tau)$, as a result of the existence and uniqueness theorem of IVPs and Theorem 4.1 in chapter 5 of \cite{hartman2002ode}. The resulting smooth maps are denoted by
\begin{equation}
    \Psi(\lambda,\tau) = (g,u),~\Gamma(g,u) = \Lambda(l)
\end{equation}
The admissible set of $(\lambda,\tau)$ can then be characterized as
\begin{equation}
    \mathcal{A} = \{(\lambda,\tau)~|~F(\lambda,\tau)=0\}
\end{equation}
where $F:~\mathbb{R}^{6}\times\mathbb{R}^{n}\rightarrow\mathbb{R}^6$ represents the boundary condition
\begin{equation}
    F(\lambda,\tau)=\Gamma\circ\Psi(\lambda,\tau)- W_+(g(l))
\end{equation}

We first examine the Jacobian of $F$. Suppose $\mathrm{rank}(F_\lambda) < 6$, then there exists a first-order perturbation $\delta\lambda \in \mathrm{ker}(F_\lambda)$ such that $\delta F = F_\lambda\delta\lambda = 0$. This means that for a fixed $\tau$, there exist extremal configurations within an arbitrarily small neighborhood of $\lambda$. This corresponds to marginally stable configurations of the robot, where the robot would move to another stable configuration after a small perturbation. Indeed, since there can be multiple different $\lambda$ corresponding to a single $\tau$ when the robot is in a potential field, there can be bifurcation points where the mapping from $\tau$ to $\lambda$ becomes not one-to-one. For the following proof, define
\begin{equation}
    \Bar{\mathcal{A}} = \{(\lambda,\tau)~|~F(\lambda,\tau)=0,~\mathrm{rank}(F_\lambda(\lambda,\tau)) = 6\}
\end{equation}

 \begin{lemma}
     $\Bar{\mathcal{A}}$ is a $n$-dimensional smooth manifold.
 \end{lemma}
\renewcommand\qedsymbol{$\square$}
\vspace{-2mm}
 \begin{proof}
    Clearly, $F$ is a smooth map. Since its Jacobian matrix $[F_\lambda~F_\tau]$ has constant rank $6$ over $\Bar{\mathcal{A}}$, $\Bar{\mathcal{A}}$ is a smooth $n$-dimensional submanifold of $\mathbb{R}^{6+n}$ by the constant-rank level set theorem (Theorem 5.12, \cite{lee_introduction_2012}). %\hfill{$\square$}
\end{proof}
\vspace{-2mm}
  
\begin{lemma}
The map $\Psi:~\Bar{\mathcal{A}}\rightarrow\Bar{\mathcal{C}}$ is a diffeomorphism.
\end{lemma}\label{homeo}
\begin{proof}
% \textit{Proof:} 
By construction, $\Psi$ is well-defined, smooth, and surjective. It remains to show that $\Psi$ is injective and has a smooth inverse. First note that $\lambda$ and $g(l)$ uniquely and smoothly depend on $(g,u)$, and the boundary value $\Lambda(l)$ is uniquely and smoothly determined by $g$ through (\ref{B.C.}). We then can define a new IVP over the interval $[l,0]$ that starts from $s=l$ and propagates back in $s$, with known initial values $(g(l),\Lambda(l))$ and a system of ODEs (\ref{canonical_eqs}) parameterized by $u$. Since these ODEs are smooth in $(g,\Lambda)$ and $u$, by the existence and uniqueness theorem of IVP and Theorem 4.1 in chapter 5 of \cite{hartman2002ode}, the solution $(g,\Lambda)$ to this problem depends uniquely and smoothly on $(g(l),\Lambda(l))$ and $u$. Therefore, $\Lambda(0)$ is uniquely and smoothly determined by $(g,u)$, and by (\ref{constitutive_law}) we conclude that $\tau$ is also uniquely and smoothly determined by $(g,u)$. Thus, we have proved that $\Psi$ is a smooth bijection and $\Psi^{-1}$ is also smooth. %\hfill{$\square$}
\end{proof}

\begin{theorem}
    $\Psi(\Bar{\mathcal{A}}) \subset \mathcal{C}$ is a smooth $n$-manifold.
\end{theorem}
\begin{proof}
By Lemma 1, $\Psi:~\mathcal{A}\rightarrow\mathcal{C}$ is a diffeomorphism, which preserves the differential structure of the smooth manifold $\Bar{\mathcal{A}}\subset\mathcal{A}$, hence $\Psi(\Bar{\mathcal{A}})$ is a smooth manifold.
%\hfill{$\square$}
\end{proof} 

The result that the configuration space is a manifold of finite dimensions is similar to the main result in \cite{bretl2014kirchhoff}. However, unlike \cite{bretl2014kirchhoff}, $\Bar{\mathcal{A}}$ is not an open subset of the Euclidean space, but an implicitly defined closed submanifold, and there is not a single global chart. Therefore, applying Euclidean space planning methods is not appropriate for our problem.
We also note that, in the proof of Lemma 2, the bijectivity of $\Psi:~\mathcal{A}\rightarrow\mathcal{C}$ only relies on the existence and uniqueness theorem of IVP and does not require $\mathcal{A}$ to be a smooth manifold. This means that marginally stable configurations can still be explored in planning by sampling on $\mathcal{A}$.

\section{Planning on Implicit Manifold} \label{sec_planning}

The result obtained in the last section naturally suggests using planning methods that work on manifolds. A simple algorithm would be planning in the ambient Euclidean space and projecting the path to the manifold. More advanced manifold planning algorithms have also been developed including both optimization-based \cite{bordalba2023manifold_collocation} and sampling-based \cite{jaillet2012atlasRRTstar,jaillet2013AtlasRRT,kingston2019IMACS} methods. We employ a modified AtlasRRT* \cite{jaillet2012atlasRRTstar} as a demonstration of the potential to apply these methods. 
% Sampling on $\mathcal{A}_{\text{stable}}\times\mathbb{R}$.

The AtlasRRT* has a similar algorithm structure to the original RRT* \cite{karaman2011optimal}. The difference is that, apart from the tree, AtlasRRT* also maintains a collection of local charts that approximates an atlas of the manifold, and the sampling and steering methods are modified based on the atlas structure. The atlas is a collection of tangent spaces of the implicit manifold acting as local charts. Each local chart has a maximum valid radius $R$ such that the tangent space approximates the manifold well within this radius. 

Consider a chart whose origin is at $x_i = (\lambda_i,\tau_i) \in \Bar{\mathcal{A}}$, we construct it as a subset of $T_{x_i}\Bar{\mathcal{A}}$ with a parameterization $x_j = \psi_i(y_j^i)$ where $y_j^i \in T_{x_i}\Bar{\mathcal{A}}$, as shown in Fig. \ref{fig.atlas}. Since $T_{x_i}\Bar{\mathcal{A}}$ is a subspace of the ambient space $\mathbb{R}^6\times\mathbb{R}^n$, we can assign to it a $(6+n)\times n$ basis $\Phi_i$ that is expressed in the ambient space and belongs to the kernel of the Jacobian matrix $[F_{\lambda}~F_{\tau}]$. In \cite{jaillet2012atlasRRTstar}, a set of orthonormal bases are used. However, since we are not sampling only in the actuation space, the following unorthogonal basis is employed to facilitate exploration:
\vspace{-1mm}
\begin{equation}
    \Phi_i = 
    \begin{bmatrix}
        -F_{\lambda}^{-1}F_{\tau} \\ I_{n\times n}
    \end{bmatrix}
    \vspace{-1mm}
\end{equation}
We can then obtain the coordinates of the tangent space elements in the ambient Euclidean space
\begin{equation}
    x_j^i = \varphi_i(y_j^i) = x_i + \Phi_iy_j^i
\end{equation}
The manifold parameterization $\psi_i$ is then obtained by solving
\vspace{-2mm}
\begin{equation}
    F(x_j) = 0
    \vspace{-1mm}
\end{equation}
using direct shooting from the initial guess $x_j^i$. Note that this is exactly solving the BVP of CR mechanics. Unlike the orthogonal projection used in \cite{jaillet2012atlasRRTstar}, our parameterization removed the orthogonal constraints to reduce the computational load induced by the BVP.

\begin{figure}[t!]
\vspace{2mm}
    \centering
    \includegraphics[width = 0.5\linewidth]{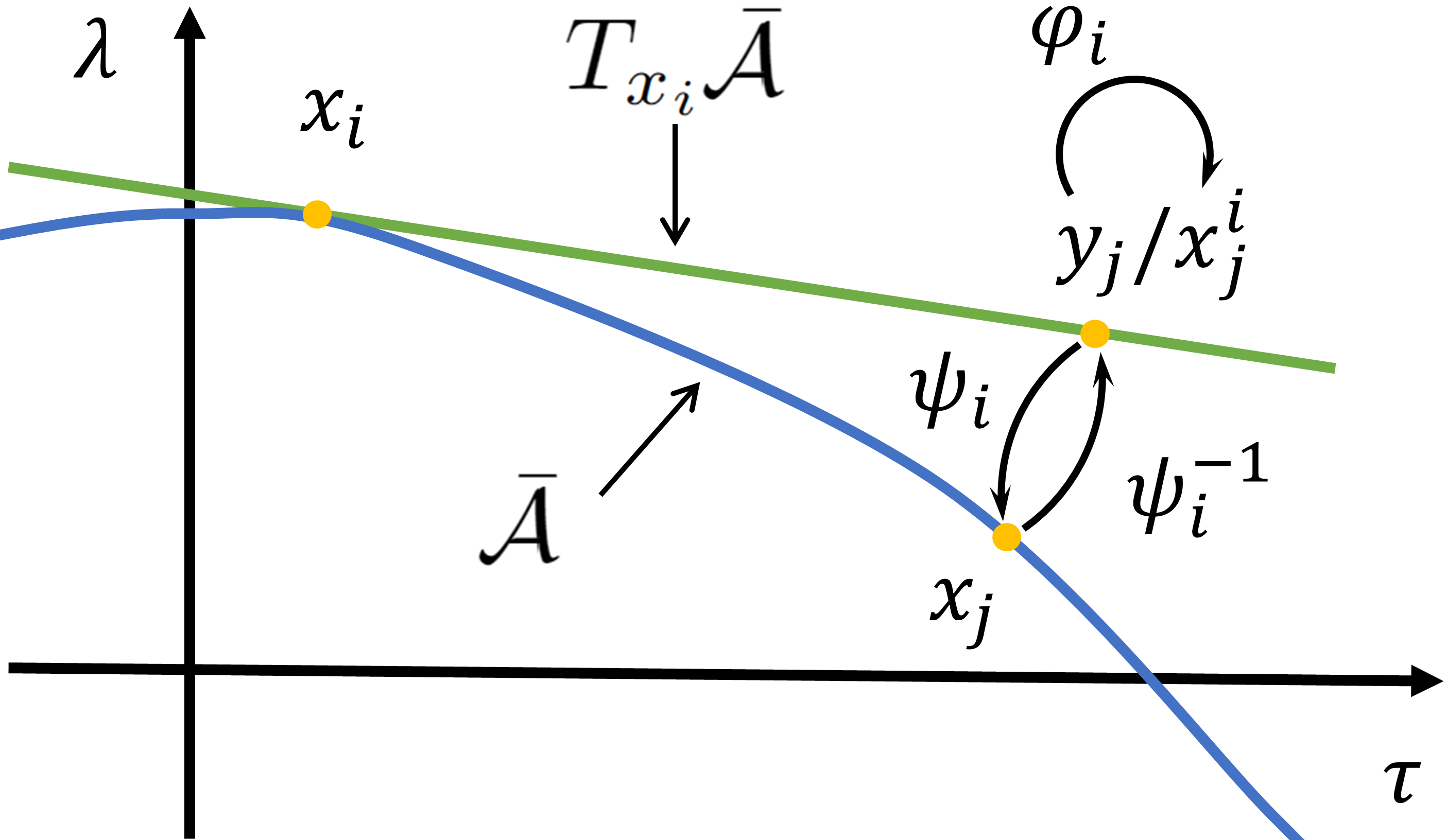}
    \vspace{-3mm}
    \caption{Illustration of the implicit manifold $\Bar{\mathcal{A}}$ and a tangent space in the ambient space.}%\ag{To fill in the white space, it would be interesting to see a ball of radius $R$. But, if that takes too much effort, then ignore.}} 
    \label{fig.atlas}
    \vspace{-6mm}
\end{figure}

To start the RRT, the atlas is initialized to a single chart with the starting configuration as the origin, and new charts are added as more samples are obtained. To sample a new configuration, an existing chart in the atlas is first selected according to the probability as follows
\vspace{-1mm}
\begin{equation}
    p_i = \frac{(\mathrm{max}(N) - N_i)^2}{\sum_i(\mathrm{max}(N) - N_i)^2}
    \vspace{-1mm}
\end{equation}
where $N_i$ is the number of times that chart $i$ was sampled. Suppose $x_c$ is the origin of the selected chart $c$, a random vector $y_{rand}$ is then generated with its length adjusted to
\vspace{-1mm}
\begin{equation}\label{edge_length}
    \|\Phi_c y_{rand}\| = \delta\beta R
    \vspace{-1mm}
\end{equation}
where $\delta$ is a random number in $[0.5,1]$ and $\beta$ is a constant number greater than $1$ to encourage exploration. The algorithm then finds the nearest node in the tree to the random sample $x_{rand} = \varphi_c(y_{rand})$, called $x_{near}$. If $x_{near}$ is in a different chart $c^\prime$, then $x_{rand}$ is orthogonally projected to this chart by $y_{rand}^{c^\prime} = (\Phi_{c^\prime}^T\Phi_{c^\prime})^{-1}\Phi_{c^\prime}^T(x_{rand} - x_{near})$ and the distance between $x_{rand}$ and $x_{near}$ is adjusted again to $\delta\beta R$. Then $x_{rand}$ is projected to the manifold $\mathcal{A}$ by $x_{new} = \psi_c(y_{rand})$, and $x_{new}$ is added to the RRT if it passes the collision check. However, before adding $x_{new}$ to the atlas, we need to determine if it is either in the current chart, in another chart, or not covered by an existing chart. To check whether a configuration is in $c$, we employ two criteria, namely, the distance from the origin of $c$ to $y$
\vspace{-1mm}
\begin{equation}\label{criterion1}
    \|\Phi_c y\| \leq R
\end{equation}
and the distance between the $y$ in the tangent space and its projection $x$ in the manifold
\begin{equation}\label{criterion2}
    \|\varphi_c(y) - x\| \leq \varepsilon
\end{equation}
The parameters $R$ and $\varepsilon$ are set to appropriately small values to ensure a good approximation of the manifold by the chart. Note that the norm used in (\ref{edge_length})-(\ref{criterion2}) is not Euclidean, since $(\lambda,\tau)$ contains values of different metrics. We calculate the norm by a robot-specific diagonal distance metric matrix $M$ to balance the weight of different values:
\vspace{-1.5mm}
\begin{equation}
    \|x\| = \sqrt{x^TMx}
    \vspace{-2mm}
\end{equation}

To reduce overlaps between different charts, a set of linear inequalities is defined for each chart
\vspace{-1.5mm}
\begin{equation}\label{inequalities}
    2y^Ty_j^c \leq (y_j^c)^Ty_j^c
    \vspace{-2mm}
\end{equation}
where $j$ is for all neighboring charts to chart $c$, and $y_j^c = (\Phi_c^T\Phi_c)^{-1}\Phi_c^T(x_j - x_c)$ is the origin of chart $j$ orthogonally projected onto chart $c$. These constraints render the chart as a convex set. 
After generating a new configuration $x_{new}$, it is checked using conditions (\ref{criterion1}),  (\ref{criterion2}), and (\ref{inequalities}) to see if it is covered by the current chart or neighboring charts. If it is not covered, a new chart is generated with $x_n$ as the origin. The new chart then recognizes all other charts whose origin is within a ball of radius $2R$ centered at $x_n$ as the neighboring charts, and updates the conditions (\ref{inequalities}) accordingly.

The implementation of AtlasRRT* is the same to the original RRT* \cite{karaman2011optimal} except the functions SAMPLE and STEER, which haven been described above. Algorithm 1 summarizes the implementation of the STEER function.

\begin{algorithm}[t]
\SetAlgoLined
\DontPrintSemicolon
\caption{The STEER function}\label{alg:steer}
\KwIn{Atlas $A$, $x_{near}$, $x_{rand}$}
\KwOut{$x_n$}
$c \gets \text{ChartIndex}(x_{rand})$\;
$c^\prime \gets \text{ChartIndex}(x_{near})$\;
\If{$c \neq c^\prime$}{
    $c \gets c^\prime$\;
    $y_{rand} \gets (\Phi_c^T\Phi_c)^{-1}\Phi_c^T(x_{rand} - x_c)$\;
    $x_{rand} \gets x_{near} + \delta\beta R\frac{\varphi_c(y_{rand})-x_{near}}{\|\varphi_c(y_{rand})-x_{near}\|}$\; 
}
$x_{new},~\text{Converge} \gets \text{Solve}(F(x)=0,~x_{rand})$\;
\If{$\text{Converge}$ \textbf{and} $\text{CollisionFree}(x_{near},~x_{new})$}{
    $y_{new} \gets \psi_c^{-1}(x_{new})$\;
    $c_{neighbor} \gets \emptyset$\;
    \If{$2y_{new}^Ty_j^c > (y_j^c)^Ty_j^c$}{
        $c_{neighbor} \gets \text{NeighborChart}(c,~x_{new},~y_{new})$\;
    }
    \eIf{$\text{IsEmpty}(c_{neighbor})$}{
        \If{$\|\Phi_c y_{new}\| \leq R$ \textbf{and} $\|\varphi_c(y_{new}) - x_{new}\| \leq \varepsilon$}{
            $c \gets \text{NewChart}(A,~x_{new})$\;
        }
    }{
        $c \gets c_{neighbor}$
    }
    $\text{AddToAtlas}(A,~c,~x_{new},~y_{new})$
}
% \vspace{-6mm}
\end{algorithm}
\setlength{\textfloatsep}{5pt}

\section{Results and Discussions} \label{sec_results}

In this section, we present simulation results for performance evaluation of the AtlasRRT* and compare it to other RRT* algorithms that sample in the ambient space. We modeled a single-segment tendon-driven CR with 4 tendons placed $90^\circ$ apart around the backbone, which has a similar design in \cite{xiao2023kinematics}. The routing of the tendons are parallel to the centerline of the backbone. Each $180^\circ$ opposing tendon pair is driven differentially  such that when one tendon is pulled the other has zero tension, generating 2 DoFs of bending motion actuated by tendon tensions. The robot also possesses 1 DoF of elongation/shortening of the backbone, resulting in 3 DoFs in total. 
The backbone is a 1 mm radius rod with $E = 50$ GPa and $G = 20$ GPa, and the tendons are 15mm away from the backbone. The maximum tendon tension is 70 N and the robot length is between 25 mm and 100 mm.
To solve the BVP using the shooting method, we used the Runge-Kutta method in the Matlab function \texttt{ode45()} for the forward integration of (\ref{canonical_eqs}) and the Levenberg-Marquardt method in \texttt{fsolve()} to find the unknown initial value $\lambda$.

For comparison, we implemented two different variants of the original RRT* algorithm. The first one, named RRT*-$\tau$, samples in $\mathbb{R}^3$ for $\tau$ and uses the $\lambda$ of the starting configuration as the initial guess of the BVP. The second one, named RRT*-$(\lambda,\tau)$, samples in $\mathbb{R}^6\times\mathbb{R}^3$ for both $\lambda$ and $\tau$, and uses them to solve the BVP. Based on simulation trials, we defined the distance metric as $M = \mathrm{diag}([10^2~10^2~10^2~0.8~0.8~0.8~1~1~10^4])$ for $(\lambda,\tau)$ in SI units. The AtlasRRT* used parameters: $R = 10$, $\varepsilon = 5$, and $\beta = 5$. The tree extension distance was 20 for RRT*-$(\lambda,\tau)$ and 7 for RRT*-$\tau$ (without $\lambda$ component). The cost was the Euclidean distance of the robot tip path.

\begin{figure*}[t!]
\vspace{2mm}
    \centering
    \includegraphics[width = 1\linewidth]{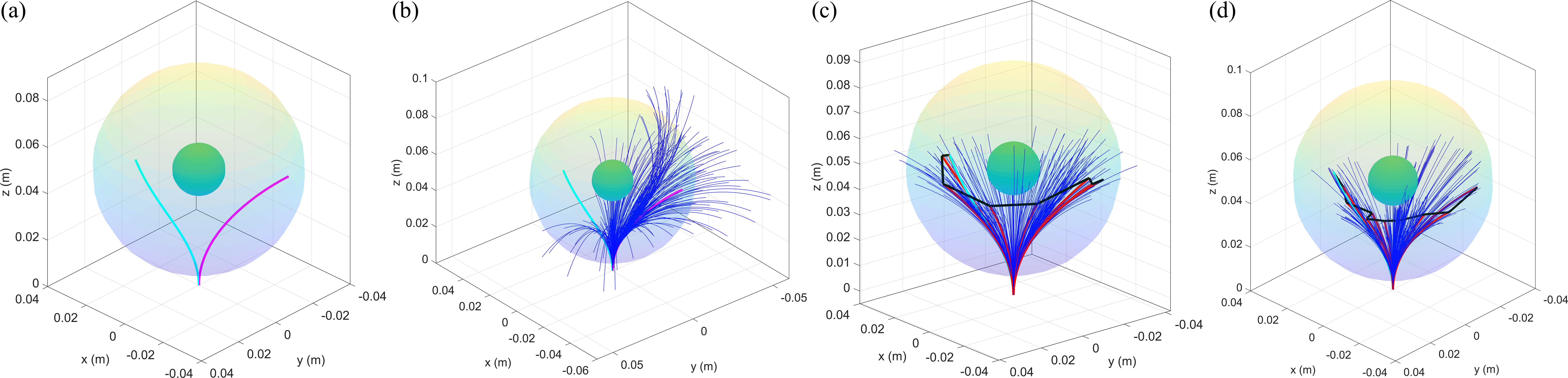}
    \vspace{-5mm}
    \caption{Results for scenario 1. The magenta and cyan curves represent the start and target configurations, respectively. The blue curves are sampled configurations, while the red curves are samples contained in the path. The solid shapes represent obstacles and the transparent shapes represent the potential field. (a) Scenario setup; (b) RRT*-$\tau$; (c) RRT*-$(\lambda,\tau)$; (d) AtlasRRT*}
    \label{fig.result-1}
    \vspace{-1mm}
\end{figure*}
\begin{figure*}[t!]
    \centering
    \includegraphics[width = 1\linewidth]{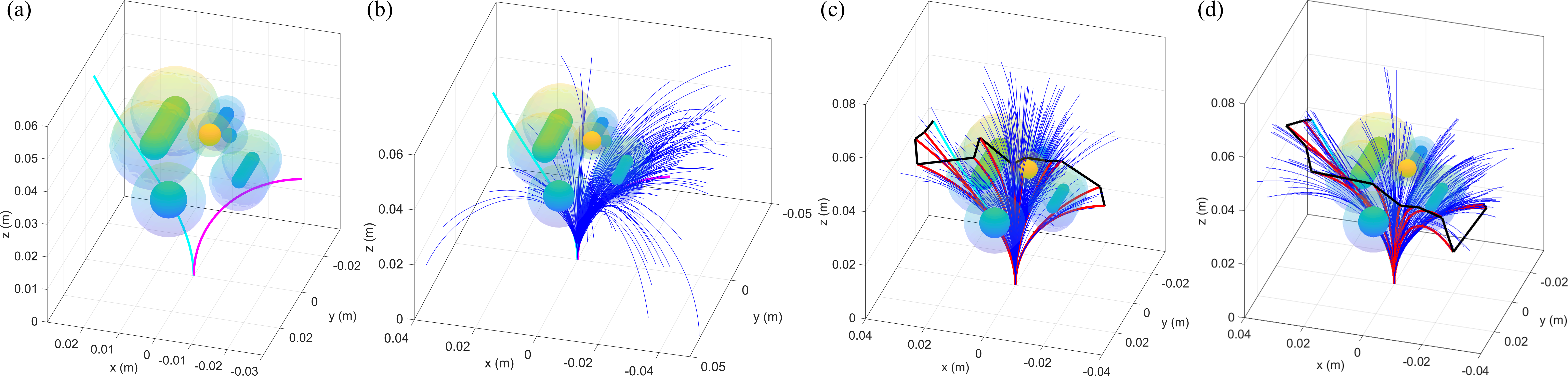}
    \vspace{-4mm}
    \caption{Results for scenario 2. The color scheme follows that of Fig. \ref{fig.result-1}. (a) Scenario setup; (b) RRT*-$\tau$; (c) RRT*-$(\lambda,\tau)$; (d) AtlasRRT*}
    \label{fig.result-2}
    \vspace{-6mm}
\end{figure*}

Two scenarios were designed to test the algorithms. In the first scenario, the environment only contains a ball centered above the robot base as the obstacle. A spherical potential field is generated concentrically to the ball to simulate an elastic ball object. The potential field is large enough to ensure the robot is always in contact with the ball. The start and target configurations of the robot have the same actuation values $\tau$ but different $\lambda$. This simple scenario aims to clearly demonstrate the difference between the algorithms. In the second scenario, the environment contains multiple ball- and capsule-shaped objects with potential fields. The start and target configurations are deflected by different objects so that the robot needs to explore contact to find a path. All algorithms generated 300 nodes in the tree before timeout. In both scenarios, each algorithm was tested 5 times with different random seeds. All simulations were implemented in Matlab and run on an 8-core 2.3 GHz processor.
% parameters

The simulation results are presented in Table. \ref{table_result}, including number of samples in the tree before finding a path, path cost, and computation time. In scenario 1, AtlasRRT* and RRT*-$(\lambda,\tau)$ both found a path, while RRT*-$\tau$ could not find a path before timeout. The total computation time by AtlasRRT* was 50\% less than RRT*-$(\lambda,\tau)$ and 66\% less than RRT*-$\tau$. 
A comparison between the sampled configurations of the three methods is shown in Fig. \ref{fig.result-1}. We observe that the sampled configurations by RRT*-$\tau$ are mostly concentrated on one side of the spherical potential field due to the fixed $\lambda$. Since the start and target configurations have the same $\tau$ value, samples close to the target in the $\tau$-space would have configurations close to the start. On the other hand, RRT*-$(\lambda,\tau)$ and AtlasRRT* explored a larger workspace and both found a path within reasonable numbers of explored configurations. However, RRT*-$(\lambda,\tau)$ required significantly more computation time, as it samples across the entire $\mathbb{R}^6\times\mathbb{R}^3$, with many samples not close to $\mathcal{A}$, often precluding convergence or requiring more iterations to converge to $\mathcal{A}$.

Similar results can be observed in scenario 2, where AtlasRRT* and RRT*-$(\lambda,\tau)$ both found a path while RRT*-$\tau$ failed again before timeout. The total computation time by AtlasRRT* was 48\% less than RRT*-$(\lambda,\tau)$ and 59\% less than RRT*-$\tau$. As shown in Fig. \ref{fig.result-2}, AtlasRRT* and RRT*-$(\lambda,\tau)$ achieved better exploration than RRT*-$\tau$. Compared to scenario 1, the computation time to sample the same amount of configurations increased by 33\% for RRT*-$\tau$, 52\% for RRT*-$(\lambda,\tau)$, and 18\% for AtlasRRT*. Both RRT*-$\tau$ and RRT*-$(\lambda,\tau)$ suffered a greater increase in computation time than AtlasRRT*. This is due to the complexity of the potential field, which reduces the likelihood of convergence using the shooting method given a bad initial guess.

\begin{table}[t!]

\caption{Averaged Simulation Results}
\label{table_result}
\vspace{-4mm}
\begin{center}
\begin{tabular}{|c|c|p{1.2cm}|p{1.2cm}|p{1.2cm}|}
\hline
\multicolumn{2}{|c|}{scenario} & {samples} & {cost (mm)} & {time (s)}\\
\hline
\multirow{3}{*}{1} & RRT*-$\tau$ & - & - & 2242.9\\ % dn = 24.82 dm = 1.06
\cline{2-5}
& RRT*-$(\lambda,\tau)$ & 49.6 & 106.6 & 1514.4\\
\cline{2-5}
& AtlasRRT* & 54.6 & 93.2 & 763.1\\
\hline
\multirow{3}{*}{2} & RRT*-$\tau$ & - & - & 2974.7 \\ % 300 nodes dn = 15.18 dm = 0.6756
\cline{2-5}
& RRT*-$(\lambda,\tau)$ & 127.2 & 125.4 & 2305.7\\ 
\cline{2-5}
& AtlasRRT* & 92.6 & 153.4 & 897.7\\ 
\hline
\end{tabular}
\end{center}
\vspace{-2.5mm}
\end{table}

We remark that bidirectional RRT can potentially improve the performance of the above algorithms. However, the target configuration is often hard to obtain for a CR in complicated environments, making bidirectional sampling less feasible. We also note that simulating the interaction of a CR with rigid bodies is possible by defining potential fields with large gradients to approximate rigid contacts. This work is also subjected to limitations to be addressed in future works. The potential field used in this work is invariable and does not capture the deformation of the elastic object in contact, making the robot configuration less likely to be stable. The shooting method used for solving robot configurations is slow and convergence is difficult, especially within complicated potential fields, resulting in significantly long computation time and less optimized paths.
\vspace{-2.5mm}

\section{CONCLUSIONS} \label{sec_conclusions}
\vspace{-1.5mm}
This paper presents a method for CR path planning with contact between the robot and elastic objects. The objects are modeled as potential fields that exert distributed forces on the robot. An analysis of the robot mechanics shows that the stable configurations of the robot are characterized by an implicit smooth manifold. A manifold planning algorithm named AtlasRRT* is then employed to solve the path planning problem. Simulations in different scenarios show that AtlasRRT* outperforms Euclidean space RRT* in terms of computational efficiency. 
Future works include using a more realistic mechanics model for elastic objects and methods such as collocation to rapidly solve robot mechanics. In addition, the proposed algorithm will be validated in real-world experiments, such as catheter-based cardiac ablations \cite{alipour2019mri} and concentric tube based hemorrhage removal \cite{gunderman2023non}.


%%%%%%%%%%%%%%%%%%%%%%%%%%%%%%%%%%%%%%%%%%%%%%%%%%%%%%%%%%%%%%%%%%%%%%%%%%%%%%%%

%%%%%%%%%%%%%%%%%%%%%%%%%%%%%%%%%%%%%%%%%%%%%%%%%%%%%%%%%%%%%%%%%%%%%%%%%%%%%%%%

%%%%%%%%%%%%%%%%%%%%%%%%%%%%%%%%%%%%%%%%%%%%%%%%%%%%%%%%%%%%%%%%%%%%%%%%%%%%%%%%

%%%%%%%%%%%%%%%%%%%%%%%%%%%%%%%%%%%%%%%%%%%%%%%%%%%%%%%%%%%%%%%%%%%%%%%%%%%%%%%%

\bibliographystyle{ieeetr}
\bibliography{references}

\end{document}